\documentclass[11pt]{article}

\usepackage[margin=1in]{geometry}
\usepackage[authoryear]{natbib}
\usepackage[hidelinks,hypertexnames=false]{hyperref}
\usepackage{graphicx}
\usepackage{mathtools}
\usepackage{amssymb}
\usepackage{amsthm}
\usepackage{booktabs}
\usepackage{array}
\usepackage{enumitem}
\usepackage{float}
\hypersetup{
  pdftitle={Common-Center Geometry and Certified Radial Reconstruction for Energy-Form Full Conformal Regions},
  pdfauthor={Yiheng Feng},
  pdfsubject={Author manuscript of the article published in Journal of Reproducible Statistics},
  pdfkeywords={62G15, conformal prediction, energy score, continuous ranked probability score, star-shaped sets, certified radial reconstruction}
}

\theoremstyle{plain}
\newtheorem{theorem}{Theorem}
\newtheorem{corollary}{Corollary}

\newtheorem{proposition}{Proposition}

\begin{document}
\title{Common-Center Geometry and Certified Radial Reconstruction for Energy-Form Full Conformal Regions}
\author{Yiheng Feng\\
School of Management, Fudan University, Shanghai, China\\
\texttt{yhfeng25@m.fudan.edu.cn}}
\date{}
\maketitle

\begin{abstract}
This note studies response-only full conformal prediction (FullCP) with no covariates or fitted predictive model, whose exact region comprises multivariate responses accepted by the conformal comparison rule. Standard conformal theory addresses validity; this note studies score-induced geometry and faithful representation when the set need not be convex. For an empirical energy-form pairwise score, the leave-one-out comparisons reduce exactly to pairwise-dissimilarity sublevel conditions that, under the stated score assumptions, share a common center; convex comparison sets then yield a star-shaped exact FullCP region. For power distances this geometry holds for $\beta\ge1$, and in the univariate $\beta=1$ specialization every nontrivial fixed-bin empirical continuous ranked probability score (CRPS) FullCP region is a nonempty closed interval. Star-shapedness permits an exact ray-wise representation, and for $1<\beta<2$ explicit data-checkable radial Lipschitz bounds convert finitely many certified directional searches into global inner and outer reconstructions with deterministic same-ray Hausdorff error guarantees. An analytic two-dimensional example shows how convexification can enlarge the exact accepted set, and a staged two-dimensional numerical study assesses certificate tightening and nonconvexity diagnostics without making a coverage or prevalence claim. In the frozen study, the focuswise bound tightened the generic bound by 3.4\% at the median, while both remained roughly two orders above a finite-grid lower proxy.  The reconstruction results are score-specific and intended for low-dimensional multivariate outputs; no runtime improvement is claimed.
\end{abstract}



\medskip
\noindent\textbf{arXiv:} \href{https://arxiv.org/abs/2608.24964}{2608.24964}

\medskip
\noindent\textbf{Publication:} Journal of Reproducible Statistics 1, Article 6 (2026).\newline
\textbf{DOI:} \href{https://doi.org/10.68365/jrepstat.2026.6}{10.68365/jrepstat.2026.6}.

\bigskip

\section{Introduction}
\label{sec:introduction}

A multivariate full conformal prediction (FullCP) procedure accepts or rejects each candidate response $t$ through an exact conformal comparison, and the collection of accepted candidates is the prediction region studied here.  Under the standard exchangeability framework, with i.i.d. sampling as a canonical special case, conformal prediction supplies its usual finite-sample marginal validity \citep{vovk_gammerman_shafer_2005_alrw,shafer_vovk_2008_tutorial}.  This paper addresses a different layer of the problem: the geometry of that exact accepted set and how to represent it faithfully.  The validity guarantee is inherited background; the contribution is score-specific exact-set geometry and deterministic reconstruction error control.

The present analysis uses the response-only special case of exchangeable conformal prediction: the observed examples are \(Y_1,\ldots,Y_m\), with no covariates and no fitted predictive model. This restriction isolates the exact-set geometry induced by the score itself from the additional candidate dependence introduced by predictive-model fitting. \citet{shafer_vovk_2008_tutorial} discuss the corresponding special case in which prediction is based on past examples without informative features of the new object. By contrast, several recent multivariate conformal methods formulate prediction for regression pairs \((X,Y)\) and/or incorporate fitted or pre-trained predictors \citep{diquigiovanni_fontana_vantini_2022, johnstone_ndiaye_2025_multioutput, gan_liu_2025}. In refitted regression-based FullCP, the trial response can also enter the fitted predictor through the augmented training data \citep{lei_gsell_rinaldo_tibshirani_wasserman_2018}, so the resulting score comparisons can depend on the trial value through model fitting as well as through direct score evaluation. Accordingly, Theorem 1 gives an exact characterization of the response-only score-induced setting studied here; it does not establish common-center geometry for refitted regression-based FullCP.

For example, at a fixed forecast context one may have exchangeable historical low-dimensional multivariate outcomes, with no predictor refit as the trial response varies.  The empirical energy-score FullCP region can then be used directly for reporting or downstream candidate-level decisions.

That distinction matters in more than one response dimension.  An exact FullCP region can be nonconvex, while a finite collection of directional boundary searches leaves uncertainty between sampled directions.  Replacing the resulting boundary by an ellipse or convex hull may be convenient, but if the exact region is nonconvex then convexification can enlarge it and add candidates rejected by the exact conformal comparison.  Such an enlargement is not claimed here to destroy marginal coverage validity; rather, it is no longer a faithful representation of the exact accepted set and can change its size, shape, and candidate-level decisions.

For the empirical energy-form pairwise score considered below, the leave-one-out algebra supplies a stronger structure than candidate-score convexity.  Each training-versus-candidate comparison becomes exactly a pairwise-dissimilarity sublevel condition.  Under the stated diagonal and attainment assumptions, all comparison regions contain a common Fr\'echet-type minimizer; when those comparison regions are convex, the nontrivial exact FullCP vote aggregate is star-shaped about that same point.  Star-shapedness is the useful middle ground: the aggregate need not be convex, but every ray from the common center has a single radial exit, so the whole region can be represented ray by ray.

Power distances $\rho_\beta(x,y)=\|x-y\|^\beta$ provide the main concrete family.  Their deterministic common-center geometry holds for $\beta\ge1$.  In one dimension at $\beta=1$, the score has the usual energy representation of the continuous ranked probability score (CRPS) \citep{gneiting_raftery_2007_scoring}; for the fixed-bin empirical construction studied by \citet[Sections~8.1 and~8.3]{toccaceli_2026_crps}, the score geometry proves that every nontrivial exact FullCP set is a nonempty closed interval.  This CRPS result is a recognizable specialization of the multivariate mechanism, not a new scoring-rule theorem.

The radial representation becomes certifiable on the regular range $1<\beta<2$.  Along each direction, the exact FullCP radius is an order statistic of the component comparison-set radii.  Explicit score-derived Lipschitz bounds on those radii then turn finitely many certified ray searches into global lower and upper radial envelopes.  The envelope width combines the one-dimensional search tolerance with a directional covering term, and the same-center radial sandwich yields deterministic Hausdorff error bounds.  Thus finite-direction reconstruction is accompanied by a global approximation certificate rather than only a plotted collection of rays.

The paper makes three contributions.  First, it gives the exact score-difference reduction, common-center comparison geometry, and star-shaped aggregation, including the exact univariate empirical-CRPS interval result.  Second, it identifies the radial order-statistic interface and, for $1<\beta<2$, explicit data-checkable radial regularity that makes finite-direction certification possible.  Third, it derives a focuswise radial certificate that is deterministically no looser than a generic center-based benchmark and assesses the resulting certificate tightness and nonconvexity diagnostics in a frozen staged two-dimensional study.  The analytic example in Section~\ref{sec:nonconvexity} shows explicitly why preserving nonconvexity can matter before the certification machinery is introduced.

The scope is deliberately limited.  The geometric statements are specific to the score structure and assumptions developed here; no generic theorem for arbitrary FullCP regions is claimed.  The paper does not add a conformal coverage theorem, and response-dependent adaptive partition questions remain outside scope.  Directional covering restricts practical certification to low-dimensional multivariate outputs, and no runtime superiority is asserted.  The numerical study is descriptive for its frozen designs rather than an estimate of how often nonconvexity occurs in a population of FullCP problems.

The standard variogram score used in multivariate forecast evaluation is not of the empirical observation-pair energy-form V-statistic form assumed in Theorem~\ref{thm:common-center}, so Theorem~\ref{thm:common-center} does not directly cover it.  This scope statement does not rule out analogous geometry for other variogram constructions.

Root finding and bisection for conformal sets predate this work \citep{ndiaye_takeuchi_2023_rootfinding}, as does multi-output reconstruction from finitely many search directions \citep{johnstone_ndiaye_2025_multioutput}.  In the latter framework, star-shapedness can be imposed as a structural condition and sampled boundary points can be fit by convex objects; here the common star center is instead derived from the score before directional approximation.  Classical radius-$R$ supporting-ball geometry \citep{balashov_golubev_2012_projection} and McShane--Whitney Lipschitz-extension formulas \citep{arnau_calabuig_erdogan_sanchezperez_2023} are used later as inherited tools after the score-specific bounds are established.  Other certified FullCP approximations use different error metrics and scopes \citep{razafindrakoto_celisse_lacaille_2026a,razafindrakoto_celisse_lacaille_2026b}.  A companion Lean~4 development checks selected discrete and algebraic components of the present results; the multivariate regularity, adaptive-partition, and coverage questions are outside that formalization.

Section~\ref{sec:score-geometry} develops the exact score-induced FullCP geometry.  Section~\ref{sec:radial-interface} specializes to power distances, isolates the empirical-CRPS interval result, and gives the exact radial order-statistic representation.  Section~\ref{sec:nonconvexity} gives the analytic nonconvex witness before Section~\ref{sec:certified-reconstruction} develops certified reconstruction for $1<\beta<2$.  Section~\ref{sec:empirical_geometry} reports the staged numerical assessment, Section~\ref{sec:limitations} records computational and endpoint limitations, and Section~\ref{sec:discussion} synthesizes the scope.  Appendices A--D contain the $\beta=1$ segment pathology, a three-point robustness witness, the detailed technical regularity calculations, and numerical protocol and reproducibility details.

\section{Exact Score-Induced Full-Conformal Geometry}
\label{sec:score-geometry}

The setup can be read as a voting rule.  For each candidate $t$, training observation $j$ contributes one comparison vote when its leave-one-out score is at least the candidate score; the exact FullCP region consists of candidates receiving enough such votes.  The notation below makes that rule explicit before exploiting the score algebra.

Let $m\ge1$, $d\ge1$, and let $y_1,\ldots,y_m\in\mathbb R^d$.  For a function
\[
\rho:\mathbb R^d\times\mathbb R^d\to\mathbb R
\]
and an empirical distribution
\[
\widehat F=\frac1m\sum_{r=1}^m\delta_{z_r},
\]
define the energy-form pairwise score
\[
S_\rho(\widehat F,y)
=
\frac1m\sum_{r=1}^m\rho(z_r,y)
-
\frac1{2m^2}\sum_{r,s=1}^m\rho(z_r,z_s).
\]
The second term uses the empirical V-statistic normalization $1/(2m^2)$; no off-diagonal normalization is used.

For a candidate $t\in\mathbb R^d$, augment the observed responses by
\[
x_i(t)=y_i,\qquad i=1,\ldots,m,
\qquad
x_{m+1}(t)=t.
\]
For $r=1,\ldots,m+1$, let
\[
\widehat F_{-r,t}
=
\frac1m
\sum_{\substack{s=1\\s\ne r}}^{m+1}
\delta_{x_s(t)},
\qquad
\alpha_r(t)
=
S_\rho\!\left(\widehat F_{-r,t},x_r(t)\right),
\]
and write
\[
\alpha(t)=\alpha_{m+1}(t).
\]
For $j=1,\ldots,m$, define the comparison region
\[
I_j=\{t:\alpha_j(t)\ge\alpha(t)\}.
\]
Thus score ties are accepted as comparison votes.  Let
\[
N(t)=\sum_{j=1}^m\mathbf 1\{t\in I_j\},
\qquad
p(t)=\frac{1+N(t)}{m+1},
\]
and use throughout the strict conformal convention
\[
\Gamma^\varepsilon=\{t:p(t)>\varepsilon\}.
\]
For $0\le\varepsilon<1$, set
\[
k=\lfloor(m+1)\varepsilon\rfloor.
\]
Thus the exact endpoint regimes are
\[
k=0
\quad\left(0\le\varepsilon<\frac1{m+1}\right)
\Longrightarrow
\Gamma^\varepsilon=\mathbb R^d,
\]
while $1\le k\le m$ is the nontrivial vote-threshold regime; in that regime the integer-valued vote count gives the exact equivalence
\[
\Gamma^\varepsilon=\{t:N(t)\ge k\}.
\]
At $\varepsilon=1$, the strict convention gives $\Gamma^1=\varnothing$.  These endpoint, tie, and threshold conventions will not be repeated inside later formal statements.

\paragraph{Candidate-score convexity is not sufficient.}
To see why the score-specific structure below is needed, temporarily replace the energy-form score by the generic reference-sample-symmetric nonconformity rule
\[
A(z;Z)
=
\frac1{|Z|}
\sum_{w\in Z}|z-w^2|
=
\frac1{|Z|}
\sum_{w\in Z}\ell(z,w),
\qquad
\ell(u,v)=|u-v^2|.
\]
The rule is symmetric in the reference multiset $Z$, and $\ell(\,\cdot\,,v)$ is convex for every fixed $v$.

Take $m=2$ and $y_1=y_2=2$.  For a candidate $t$, its score is
\[
\widetilde\alpha(t)
=
A(t;\{2,2\})
=
|t-4|,
\]
which is convex, whereas the two training scores coincide:
\[
\widetilde\alpha_1(t)
=
\widetilde\alpha_2(t)
=
A(2;\{2,t\})
=
1+\frac12|2-t^2|.
\]
A direct piecewise calculation gives
\[
\widetilde\alpha_j(t)\ge\widetilde\alpha(t)
\quad\Longleftrightarrow\quad
t\le-4
\quad\text{or}\quad
t\ge2.
\]
Thus the corresponding full-conformal $p$-value is
\[
\widetilde p(t)
=
\begin{cases}
1,
&
t\in(-\infty,-4]\cup[2,\infty),\\[1mm]
1/3,
&
-4<t<2.
\end{cases}
\]
At $\varepsilon=1/2$, using the strict convention
$\widetilde p(t)>\varepsilon$,
\[
\widetilde\Gamma^{1/2}
=
(-\infty,-4]\cup[2,\infty),
\]
which is disconnected.  Equality holds at $t=-4$ and $t=2$, so both endpoints are included because comparison votes use $\ge$.

This example is not asserted to arise from a proper distributional scoring rule.  Its role is narrower: candidate-score convexity is not sufficient for connected FullCP regions, even within the empirical-average form with a loss convex in its candidate argument.

The following result isolates the stronger algebraic and geometric structure that is sufficient for the energy-form score.

Before stating the result, define
\[
d_j=\sum_{i\ne j}\rho(y_j,y_i),
\qquad
F(t)=\sum_{i=1}^m\rho(t,y_i).
\]

\begin{theorem}[Exact common-center geometry for empirical energy-form FullCP]
\label{thm:common-center}
Assume $\rho$ is symmetric and has constant diagonal $\rho(x,x)=q$.  Then:
\begin{enumerate}[label=(\roman*),leftmargin=2.4em,itemsep=0.35em,topsep=0.3em]
\item for every $j$ and candidate $t$,
\[
\alpha_j(t)-\alpha(t)
=
\frac{m+1}{m^2}
\left[d_j-\sum_{i\ne j}\rho(t,y_i)\right],
\qquad
I_j=\left\{t:\sum_{i\ne j}\rho(t,y_i)\le d_j\right\};
\]
\item if additionally $\rho(x,y)\ge q$ for all $x,y$ and $F$ attains a global minimizer $c$, then $c\in I_j$ for every $j$;
\item if additionally every $I_j$ is convex, then for every nontrivial threshold $1\le k\le m$, the exact FullCP region $\Gamma^\varepsilon=\{t:N(t)\ge k\}$ is star-shaped about $c$.
\end{enumerate}
\end{theorem}

\begin{proof}
Fix $j$ and put
\[
B_j=\sum_{\substack{r,s=1\\ r,s\ne j}}^m\rho(y_r,y_s),
\qquad
T_j(t)=\sum_{i\ne j}\rho(t,y_i).
\]
Leaving $t$ out of the augmented sample gives
\[
\alpha(t)
=
\frac1m\{\rho(t,y_j)+T_j(t)\}
-
\frac1{2m^2}\{B_j+2d_j+q\},
\]
whereas leaving $y_j$ out gives
\[
\alpha_j(t)
=
\frac1m\{\rho(t,y_j)+d_j\}
-
\frac1{2m^2}\{B_j+2T_j(t)+q\}.
\]
Subtracting yields part (i).

For part (ii), global minimality of $c$ gives
\[
F(c)\le F(y_j)=d_j+q,
\]
while $\rho(c,y_j)\ge q$.  Hence
\[
\sum_{i\ne j}\rho(c,y_i)
=F(c)-\rho(c,y_j)
\le d_j,
\]
so $c\in I_j$ for all $j$.

For part (iii), let $t\in\Gamma^\varepsilon$.  Then $t$ belongs to at least $k$ comparison sets, and the same sets contain $c$.  Convexity places every point of $[c,t]$ in each of those $k$ sets, so every point of the segment receives at least $k$ votes.  Thus $[c,t]\subseteq\Gamma^\varepsilon$.
\end{proof}

The theorem separates three logically distinct ingredients.  Symmetry and a constant diagonal suffice for the exact comparison identity.  The diagonal lower bound and attainment of $F$ give a common point of all comparison regions.  Convexity of the individual comparison regions is needed only for the final segment argument.  The preceding counterexample shows that convexity of the candidate score itself is not a substitute for this comparison-set structure.  The resulting $k$-coverage set can be star-shaped without being convex; Section~\ref{sec:nonconvexity} gives an explicit analytic example.

\section{Power Distances, Empirical CRPS, and Radial Representation}
\label{sec:radial-interface}

Consider the power-distance dissimilarity
\[
\rho_\beta(x,y)=\|x-y\|^\beta.
\]
This family gives a concrete deterministic geometry and, on a narrower parameter range, the conventional energy-score interpretation.

\begin{corollary}[Common-center star-shapedness for power-distance scores]
\label{cor:power}
For $\beta\ge1$, define
\[
F_\beta(t)=\sum_{i=1}^m\|t-y_i\|^\beta,
\qquad
G_j(t)=\sum_{i\ne j}\|t-y_i\|^\beta,
\qquad
d_j=\sum_{i\ne j}\|y_j-y_i\|^\beta.
\]
Then $F_\beta$ attains a global minimum, each comparison set is
\[
I_j=\{t:G_j(t)\le d_j\},
\]
and every $I_j$ contains any chosen minimizer $c$ of $F_\beta$.  Consequently, for every nontrivial threshold $1\le k\le m$, the exact FullCP region is star-shaped about $c$.
\end{corollary}

\begin{proof}
For $\beta\ge1$, $x\mapsto\|x-y\|^\beta$ is convex for fixed $y$, while $\rho_\beta$ is symmetric, has diagonal $0$, and is nonnegative.  The function $F_\beta$ is continuous and coercive, hence attains a global minimum, and each $G_j$ is convex.  Theorem~\ref{thm:common-center} applies with $q=0$.
\end{proof}

\paragraph{Finite exits and the $m=1$ degeneracy.}
When $m=1$, $G_1$ is the empty sum and $d_1=0$, so $I_1=\mathbb R^d$ and finite-radius reconstruction is degenerate.  When $m\ge2$, each $G_j$ contains at least one coercive term and $I_j=\{G_j\le d_j\}$ is bounded, so every ray from a common center has a finite exit.

For $d=1$ and $\beta=1$, this score is exactly the fixed-bin empirical CRPS score used by \citet[Section~8.1]{toccaceli_2026_crps}.  The univariate case deserves a separate formal statement because it gives the exact interval property for that construction.  Write the ordered observations as
\[
y_{(1)}\le\cdots\le y_{(m)}
\]
and define the full empirical median set
\[
\mathcal M
=
\arg\min_t F_1(t)
=
\left[
y_{(\lceil m/2\rceil)},
y_{(\lfloor m/2\rfloor+1)}
\right].
\]

For clarity, the augmented fixed-bin construction sets
\[
x_i(t)=y_i\quad(i\le m),\qquad x_{m+1}(t)=t,
\qquad
\widehat F_{-r,t}=\frac1m\sum_{\substack{s=1\\s\ne r}}^{m+1}\delta_{x_s(t)}.
\]
Its score is
\[
\operatorname{CRPS}(\widehat F_{-r,t},x_r(t))
=\frac1m\sum_{\substack{s=1\\s\ne r}}^{m+1}|x_s(t)-x_r(t)|
-\frac1{2m^2}\sum_{\substack{s,\ell=1\\s,\ell\ne r}}^{m+1}|x_s(t)-x_\ell(t)|.
\]
This is exactly the energy-form score with $\rho(x,y)=|x-y|$ and the same empirical V-statistic normalization.

\begin{corollary}[Exact interval property for univariate empirical CRPS FullCP]
\label{cor:crps-interval}
Let $d=1$ and $\beta=1$.  For every nontrivial threshold $1\le k\le m$, the fixed-bin empirical CRPS FullCP region is a nonempty closed interval containing $\mathcal M$.  For $m\ge2$ the interval is compact; for $m=1$ it is $\mathbb R$.
\end{corollary}

\begin{proof}
Theorem~\ref{thm:common-center} applies separately to every $c\in\mathcal M$, so every point of $\mathcal M$ belongs to every comparison region.  Each $I_j$ is a closed interval because it is a closed convex sublevel set in $\mathbb R$.  For $1\le k\le m$,
\[
\Gamma^\varepsilon
=
\bigcup_{\substack{S\subseteq\{1,\ldots,m\}\\ |S|=k}}
\ \bigcap_{j\in S}I_j,
\]
so $\Gamma^\varepsilon$ is closed and is nonempty because every point of $\mathcal M$ receives all $m$ votes.  It is star-shaped about every $c\in\mathcal M$ by Theorem~\ref{thm:common-center}; a nonempty star-shaped subset of $\mathbb R$ is an interval.  For $m\ge2$ the comparison sets are bounded and $k\ge1$, hence the interval is bounded and therefore compact.  The $m=1$ case is the degeneracy noted above.
\end{proof}

This proves the single-interval property reported for Toccaceli's fixed-bin empirical-CRPS construction.  The interval conclusion follows from the deterministic comparison geometry and does not by itself extend to proper distributional scoring rules in general.

\paragraph{Scoring-rule range.}
Under the conventional distributional energy-score interpretation, with
\[
\mathcal P_\beta
=
\left\{P:\int\|x\|^\beta\,dP(x)<\infty\right\},
\]
the power-distance energy score is strictly proper on the appropriate finite-moment class for $0<\beta<2$ \citep{gneiting_raftery_2007_scoring}.  At $\beta=2$ it is proper but not strictly proper and its expected score depends only on the forecast mean; for $\beta>2$ the Euclidean distance-power energy-form score is not proper.  Thus strict propriety overlaps the deterministic star-shaped geometry on $1\le\beta<2$.  These are classical scoring-rule facts and do not enlarge the deterministic claims above.

\paragraph{Lower-boundary sharpness.}
The deterministic conclusion should not be extended below $\beta=1$.  With
\[
m=3,
\qquad
\beta=\frac12,
\qquad
(y_1,y_2,y_3)=(-1,0,2),
\]
the unique Fr\'echet minimizer is $c=0$, yet at $\varepsilon=0.8$ the exact prediction set contains $0$ and $1$ while rejecting $0.5$.  Thus it is not star-shaped about $c$.

For the remainder of the radial development, assume the comparison sets are star-shaped about a common center $c$ and define, for $u\in\mathbb S^{d-1}$,
\[
R_j(u)=\sup\{r\ge0:c+ru\in I_j\}.
\]
For a nontrivial threshold, define likewise
\[
R_\varepsilon(u)=\sup\{r\ge0:c+ru\in\Gamma^\varepsilon\},
\]
and let $R_{(1)}(u)\le\cdots\le R_{(m)}(u)$ be the ascending order statistics of the component exits.

\begin{proposition}[Radial order-statistic representation]
\label{prop:radial-order}
For $1\le k=\lfloor(m+1)\varepsilon\rfloor\le m$,
\[
R_\varepsilon(u)=R_{(m-k+1)}(u)
\qquad
\text{for every }u\in\mathbb S^{d-1}.
\]
\end{proposition}

\begin{proof}
Along a fixed ray, star-shapedness makes the membership set of each $I_j$ an initial interval with supremum $R_j(u)$.  The supremum radius at which at least $k$ votes remain is therefore the $k$th-largest component exit, equivalently the $(m-k+1)$st ascending order statistic.
\end{proof}

The endpoint conventions remain those fixed in Section~\ref{sec:score-geometry}: $k=0$ gives $\Gamma^\varepsilon=\mathbb R^d$ and hence $R_\varepsilon\equiv+\infty$, while $\varepsilon=1$ gives the empty set.  For power distances with $m=1$, $I_1=\mathbb R^d$ and $R_1\equiv+\infty$, so finite radial reconstruction begins at $m\ge2$.  At $\beta=1$, exact common-center star-shaped geometry remains valid but unconditional global radial Lipschitzness can fail; Appendix A gives an exact discontinuous example.

\section{Why Preserving Nonconvexity Can Matter}
\label{sec:nonconvexity}

For this section, write $\Gamma=\Gamma^\varepsilon$.  A two-dimensional analytic example makes the distinction between star-shapedness and convexity explicit before the reconstruction machinery is introduced.  Let
\[
d=2,
\qquad
m=2,
\qquad
\beta=\frac32,
\]
with
\[
y_1=(-1,0),
\qquad
y_2=(1,0),
\qquad
\varepsilon=\frac12.
\]
Then $k=\lfloor3/2\rfloor=1$.  For $j=1$, the comparison objective contains only the distance to $y_2$, and $d_1=\|y_1-y_2\|^{3/2}=2^{3/2}$.  Therefore
\[
I_1=B(y_2,2),
\qquad
I_2=B(y_1,2),
\]
and the one-vote prediction region is exactly
\[
\Gamma^\varepsilon=I_1\cup I_2.
\]
The origin is a common center.  With
\[
u(\theta)=(\cos\theta,\sin\theta),
\]
the two ball exits are
\[
\cos\theta+\sqrt{\cos^2\theta+3}
\quad\text{and}\quad
-\cos\theta+\sqrt{\cos^2\theta+3},
\]
so the union has exact radial function
\[
R(\theta)
=
|\cos\theta|+\sqrt{\cos^2\theta+3}.
\]
This set is star-shaped about $0$ but nonconvex:
\[
(-1,2),(1,2)\in\Gamma^\varepsilon,
\qquad
(0,2)\notin\Gamma^\varepsilon.
\]

Its convex hull is the stadium
\[
\operatorname{conv}\Gamma^\varepsilon
=
[-e_1,e_1]+B(0,2).
\]
Direct geometry gives the exact radial convexification discrepancy
\[
\|R_{\operatorname{conv}\Gamma}-R_\Gamma\|_\infty
=2-\sqrt3,
\]
and the exact Hausdorff discrepancy
\[
d_H(\Gamma,\operatorname{conv}\Gamma)
=\sqrt5-2.
\]
The latter is attained at the midpoint of the top or bottom flat portion of the stadium boundary, whose distance to either radius-$2$ component ball is $\sqrt5-2$.  These positive discrepancies are properties of this analytic witness; they do not imply that convex reconstruction is inappropriate in every conformal problem.  In particular, because the convex hull is a superset of the exact accepted set, no claim is made that this enlargement itself destroys marginal coverage validity.  The point is fidelity: convexification adds candidates rejected by the exact comparison and therefore changes the exact FullCP region.

For a same-scale comparison, the analytic union has area and equivalent-area radius
\[
A_\Gamma=\frac{16\pi}{3}+2\sqrt3,
\qquad
r_{\rm eq}=\sqrt{\frac{16}{3}+\frac{2\sqrt3}{\pi}}
\approx2.5369255259.
\]
Thus its radial gap $2-\sqrt3$ has normalized magnitude
\[
\frac{2-\sqrt3}{r_{\rm eq}}\approx0.1056196525.
\]
This analytic example is a stress-test and existence illustration: convexification can be materially nonfaithful, but it makes no representative-magnitude claim.  In contrast, conditional on \texttt{ROBUST\_WITNESS}, the frozen numerical study in Section~\ref{sec:empirical_geometry} had normalized radial-gap median $0.00137$, 90th percentile $0.00626$, and maximum $0.01770$.  The frozen DGPs therefore had often-detectable but usually shallow nonconvexity.  At these observed magnitudes convexification may be negligible for coarse visualization or coarse summaries, while exact boundary-near candidate membership can nevertheless change.

\begin{figure}[H]
\centering
\begin{minipage}[t]{0.475\linewidth}
  \centering
  \includegraphics[width=\linewidth]{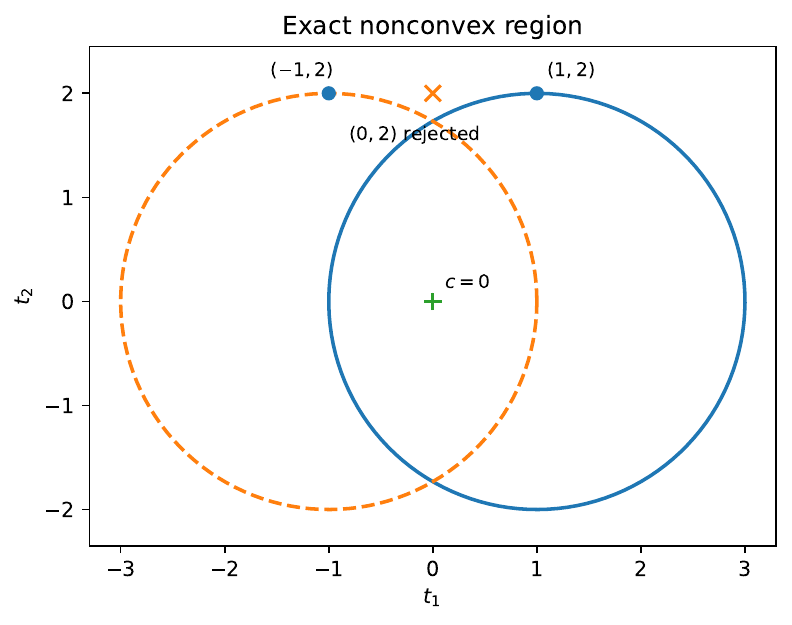}
  \par\smallskip\textbf{(A)}
\end{minipage}\hfill
\begin{minipage}[t]{0.475\linewidth}
  \centering
  \includegraphics[width=\linewidth]{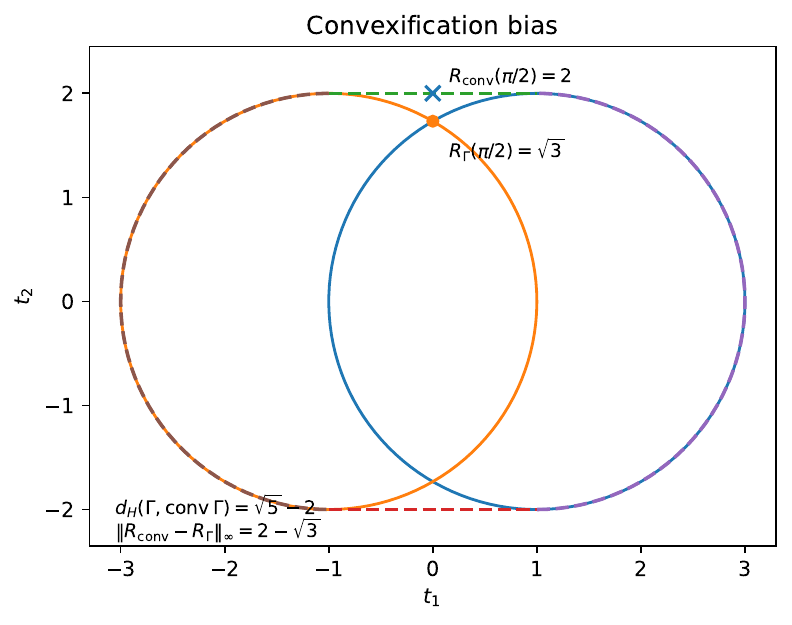}
  \par\smallskip\textbf{(B)}
\end{minipage}
\caption{%
Analytic nonconvex example with $m=2$, $d=2$, $\beta=3/2$, $y_1=(-1,0)$, $y_2=(1,0)$, and $\varepsilon=1/2$.
(A) The exact FullCP region is the union of two radius-$2$ balls, is star-shaped about $0$, and is nonconvex: $(-1,2)$ and $(1,2)$ are accepted while $(0,2)$ is rejected.
(B) The convex hull is $[-e_1,e_1]+B(0,2)$, giving exact discrepancies $\|R_{\operatorname{conv}\Gamma}-R_\Gamma\|_\infty=2-\sqrt3$ and $d_H(\Gamma,\operatorname{conv}\Gamma)=\sqrt5-2$.
}
\label{fig:analytic-nonconvex-geometry}
\end{figure}

The nonconvexity is not an artifact of using only two observations.  In the three-point example of Appendix B with $m=3$, $\beta=3/2$, and $k=2$, two symmetric endpoints are accepted while their midpoint is rejected, with nonzero comparison margins.  The numerical margins are recorded only in the appendix.  The computational implication is that the natural exact target is the star-shaped radial function, rather than a convex surrogate, which motivates the certified reconstruction developed next.

\section{Certified Radial Reconstruction for \texorpdfstring{$1<\beta<2$}{1<beta<2}}
\label{sec:certified-reconstruction}

We now restrict to $m\ge2$, $1<\beta<2$, and a nontrivial threshold $1\le k\le m$.  The goal is to turn the exact radial representation into a global certificate from finitely many ray searches.  Throughout,
\[
B(a,R)=\{x\in\mathbb R^d:\|x-a\|\le R\}
\]
denotes the closed Euclidean ball.

\subsection{Explicit radial regularity}

For each $j$, recall
\[
G_j(t)=\sum_{i\ne j}\|t-y_i\|^\beta,
\qquad
I_j=\{t:G_j(t)\le d_j\},
\]
and define
\[
\mu_{-j}=\frac1{m-1}\sum_{i\ne j}y_i,
\qquad
r_j=\left(\frac{d_j}{m-1}\right)^{1/\beta}.
\]
Jensen's inequality gives $I_j\subseteq B(\mu_{-j},r_j)$.  For a chosen common center $c$, set
\[
H_j=\|c-\mu_{-j}\|+r_j,
\]
so $I_j\subseteq B(c,H_j)$ and $H_j$ provides a finite radial-search bound.  If $d_j>0$, define the focuswise bounds
\[
U_{ij}
=
\min\left\{
r_j+\|\mu_{-j}-y_i\|,
d_j^{1/\beta}
\right\},
\]
\[
\widetilde M_j
=
\beta\sum_{i\ne j}U_{ij}^{\beta-1},
\qquad
\widetilde\lambda_j
=
\beta(\beta-1)\sum_{i\ne j}U_{ij}^{\beta-2},
\]
and
\[
\widetilde\ell_j=
\begin{cases}
0,&d_j=0,\\
2\widetilde M_j/\widetilde\lambda_j,&d_j>0.
\end{cases}
\qquad
\widetilde L=\max_j\widetilde\ell_j.
\]
The $d_j=0$ branch is separate, so no negative power is evaluated at zero.

For the generic center-based benchmark used in the deterministic tightening comparison, define for $d_j>0$
\[
D_{ij}=H_j+\|c-y_i\|,
\qquad
M_j=\beta\sum_{i\ne j}D_{ij}^{\beta-1},
\qquad
\lambda_j=\beta(\beta-1)\sum_{i\ne j}D_{ij}^{\beta-2},
\qquad
\ell_j=\frac{2M_j}{\lambda_j},
\]
with $\ell_j=0$ when $d_j=0$, and let $L_{\rm old}=\max_j\ell_j$.  The symbol $L_{\rm old}$ denotes this generic center-based bound.

\begin{proposition}[Explicit Lipschitz regularity of power-distance radial exits]
\label{prop:radial-regularity}
Under the conditions above, with chord distance on $\mathbb S^{d-1}$,
\[
\operatorname{Lip}(R_j)\le\widetilde\ell_j\le\ell_j
\qquad(j=1,\ldots,m),
\]
and
\[
\operatorname{Lip}(R_\varepsilon)\le\widetilde L\le L_{\rm old}.
\]
\end{proposition}

\begin{proof}
The mechanism has four steps; Appendix C is the unique location for the full calculations.  First, the Jensen ball and the individual summands of $G_j(t)\le d_j$ give $\|t-y_i\|\le U_{ij}$ throughout $I_j$.  Second, these focuswise bounds control the gradient and give a restricted strong-convexity inequality on each comparison set.  Completing the square at a boundary point then yields a containing tangent ball of uniform radius $\widetilde M_j/\widetilde\lambda_j$.  Third, a radius-$R$ containing ball that contains the common center has a radial exit that is $2R$-Lipschitz in chord distance; the exact supporting-ball intersection therefore gives $\operatorname{Lip}(R_j)\le2\widetilde M_j/\widetilde\lambda_j$.  When $d_j=0$, all observations coincide for that comparison and $R_j\equiv0$.

For dominance, $U_{ij}\le D_{ij}$.  Because $\beta-1>0$ and $\beta-2<0$, this gives $\widetilde M_j\le M_j$ and $\widetilde\lambda_j\ge\lambda_j$, hence $\widetilde\ell_j\le\ell_j$, including the separate zero branch.  Finally Proposition~\ref{prop:radial-order} expresses $R_\varepsilon$ as an order statistic of the $R_j$; order statistics are $1$-Lipschitz in the coordinatewise sup norm, yielding $\operatorname{Lip}(R_\varepsilon)\le\max_j\widetilde\ell_j=\widetilde L$ and therefore $\widetilde L\le L_{\rm old}$.
\end{proof}

The key geometry is that restricted strong convexity supplies uniformly curved containing balls, and a common point inside those balls prevents their ray exits from changing arbitrarily fast with direction.  Focus crossings, the tangent-ball calculation, the normal-cone and metric-projection argument for exact ball intersection, and the factor-$2$ proof are all given in Appendix C.

\subsection{Finite-direction radial envelopes}

Let
\[
\mathcal U=\{u_\ell\}_{\ell=1}^N\subset\mathbb S^{d-1}
\]
be a nonempty finite directional set.  Suppose certified radial intervals satisfy
\[
a_\ell\le R_\varepsilon(u_\ell)\le b_\ell,
\qquad
b_\ell-a_\ell\le\delta.
\]
Define the covering radius
\[
h_{\mathcal U}
=
\sup_{u\in\mathbb S^{d-1}}
\min_\ell\|u-u_\ell\|
\]
and the radial envelopes
\[
\underline R(u)
=
\max\left\{0,\max_\ell\left[a_\ell-\widetilde L\|u-u_\ell\|\right]\right\},
\]
\[
\overline R(u)
=
\min_\ell\left[b_\ell+\widetilde L\|u-u_\ell\|\right].
\]

\begin{theorem}[Certified finite-direction radial envelopes]
\label{thm:radial-envelopes}
For every $u\in\mathbb S^{d-1}$,
\[
\underline R(u)\le R_\varepsilon(u)\le\overline R(u),
\]
and
\[
\|\overline R-\underline R\|_\infty
\le
\delta+2\widetilde Lh_{\mathcal U}.
\]
\end{theorem}

\begin{proof}
Proposition~\ref{prop:radial-regularity} gives, for every sampled direction $u_\ell$,
\[
a_\ell-\widetilde L\|u-u_\ell\|
\le R_\varepsilon(u)
\le b_\ell+\widetilde L\|u-u_\ell\|.
\]
Maximizing the lower bounds, together with $R_\varepsilon\ge0$, and minimizing the upper bounds gives the sandwich.  Choosing a sampled direction within $h_{\mathcal U}$ of $u$ gives the stated width bound; Appendix C records the full two-case calculation.  These are the interval-certified counterparts of the classical McShane--Whitney extremal Lipschitz-extension formulas \citep{arnau_calabuig_erdogan_sanchezperez_2023}.
\end{proof}

\subsection{Set certification and implementation}

Define
\[
\underline\Gamma
=
\{c+ru:u\in\mathbb S^{d-1},\ 0\le r\le\underline R(u)\},
\]
and
\[
\overline\Gamma
=
\{c+ru:u\in\mathbb S^{d-1},\ 0\le r\le\overline R(u)\}.
\]

\begin{corollary}[Certified inner and outer FullCP regions]
\label{cor:set-certification}
The finite-direction envelopes satisfy
\[
\underline\Gamma\subseteq\Gamma^\varepsilon\subseteq\overline\Gamma.
\]
Moreover,
\[
d_H(\underline\Gamma,\Gamma^\varepsilon)
\le
\|R_\varepsilon-\underline R\|_\infty
\le
\delta+2\widetilde Lh_{\mathcal U},
\]
\[
d_H(\Gamma^\varepsilon,\overline\Gamma)
\le
\|\overline R-R_\varepsilon\|_\infty
\le
\delta+2\widetilde Lh_{\mathcal U},
\]
and
\[
d_H(\underline\Gamma,\overline\Gamma)
\le
\|\overline R-\underline R\|_\infty
\le
\delta+2\widetilde Lh_{\mathcal U}.
\]
\end{corollary}

\begin{proof}
The radial sandwich gives the set inclusions directly.  For two nested same-center radial sets, any point of the larger set can be moved along its own ray to the smaller radial boundary, so the directed distance is bounded by the radial-function gap in that direction.  Taking suprema gives the three inequalities; Appendix C records the full derivation.
\end{proof}

Operationally, the directional intervals require no new per-ray search algorithm.  For fixed $u$, membership along $c+ru$ is an initial interval for each comparison set, so $N(c+ru)$ is nonincreasing in $r$.  The finite bound $H_{\max}=\max_jH_j$ brackets every component exit.  If the vote threshold still holds at $H_{\max}$, the exact aggregate exit is $H_{\max}$; otherwise bisection between the accepted center and rejected outer radius produces a certified bracket of width at most $\delta$, using the root-search architecture of \citet{ndiaye_takeuchi_2023_rootfinding} and \citet{johnstone_ndiaye_2025_multioutput}.  A numerically chosen center $\tilde c$ is sufficient whenever all inequalities $G_j(\tilde c)\le d_j$ are themselves certified; the search radii $H_j$ must then be recomputed, whereas $U_{ij}$ and the tightened regularity constants do not depend on the chosen common center; $L_{\rm old}$ remains center-dependent if retained for comparison.  No runtime improvement is claimed.

The same analytic example from Section~\ref{sec:nonconvexity} illustrates the certificate after the regularity result is available.

\begin{figure}[H]
\centering
\includegraphics[width=0.76\linewidth]{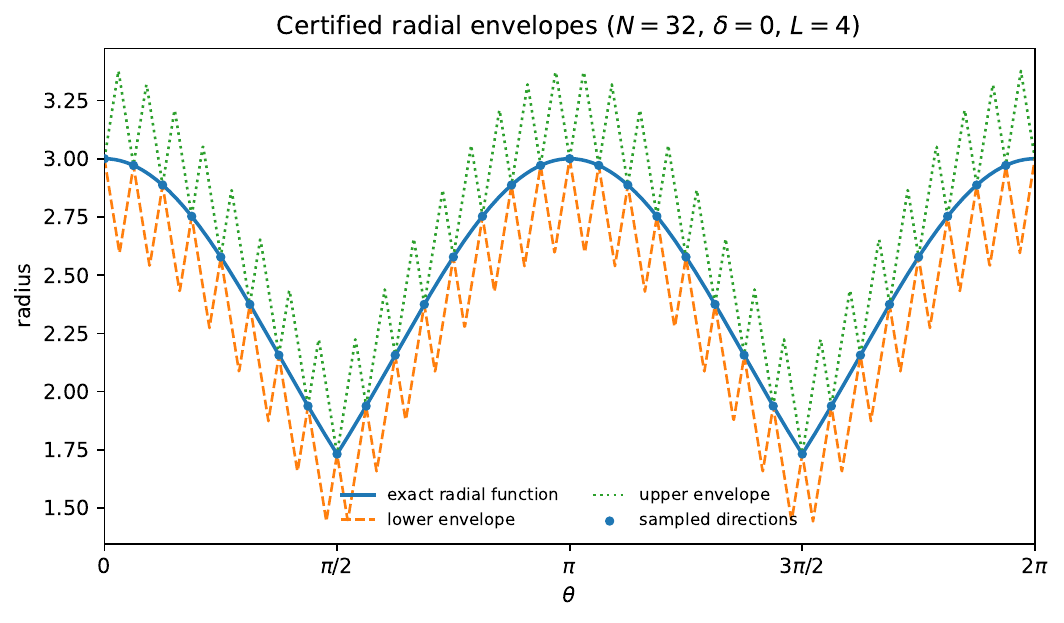}
\caption{%
Certified radial envelopes for the analytic example of Figure~\ref{fig:analytic-nonconvex-geometry}.  The rendering uses $N=32$ equally spaced directions, exact sampled radii ($\delta=0$), and the rigorous geometry-derived bound $L=4$; the envelopes contain the exact radial function.  For this example the generic center-based bound is $L_{\rm old}=16$, Proposition~\ref{prop:radial-regularity} gives $\widetilde L=8$, and direct analysis gives the sharper $L_{\rm direct}=3/2$.  The radius axis is truncated for legibility (visible tick range approximately 1.50--3.25); printed vertical separation should therefore not be interpreted as proportional to the full radial scale.
}
\label{fig:analytic-certificate}
\end{figure}

For this example, $U_{12}=U_{21}=2$, so $\widetilde M_j=(3/2)\sqrt2$, $\widetilde\lambda_j=3/(4\sqrt2)$, and $\widetilde L=2\widetilde M_j/\widetilde\lambda_j=8$.  The radius-$2$ ball geometry itself gives the rigorous bound $L=4$ used in Figure~\ref{fig:analytic-certificate}.  Finally, writing the exact radius as a function of $s=|u_1|\in[0,1]$ gives
\[
f(s)=s+\sqrt{s^2+3},
\qquad
f'(s)=1+\frac{s}{\sqrt{s^2+3}}\le\frac32,
\]
and $u\mapsto|u_1|$ is $1$-Lipschitz in chord distance, so $L_{\rm direct}=3/2$.  The sequence
\[
16\ \longrightarrow\ 8\ \longrightarrow\ 4\ \longrightarrow\ \frac32
\]
uses successively more problem-specific information; in particular, the tightened general certificate gives $8$, not $4$ or $3/2$.

The reconstruction results stop at $1<\beta<2$.  Comparison-set convexity itself holds for every $\beta\ge1$, but Appendix A shows that exact common-center geometry at $\beta=1$ does not imply unconditional global radial Lipschitzness.

\section{Numerical Assessment}
\label{sec:empirical_geometry}

The numerical study assesses certificate tightness and the detectability and radial magnitude of nonconvexity under frozen two-dimensional designs; it is not a coverage experiment and its detection frequencies are not estimates of population prevalence.  The study was staged.  The initial smoke run used 5 independent clouds per $(m,\text{data-generating process (DGP)})$ cell, 40 total, and expansion to 20 independent clouds per cell was allowed only after the seven advance-declared engineering and numerical gates listed in Appendix~\ref{app:numerical-protocol} passed; soft flags did not automatically block expansion.  After those gates passed, the expanded-production $\beta/q_{\rm vote}$ grid, seed mapping, and no-post-hoc rule were frozen before any expanded-run results were observed.  The original five clouds in each cell were retained in the final 20, so the final 160 clouds include the original 40 smoke clouds.

The 20 clouds per $(m,\mathrm{DGP})$ cell were the staged descriptive expansion size, not a target-precision sample size.  For each $m\in\{10,25\}$, the final dataset contained 20 independent clouds from each of four designs---an isotropic Gaussian, a rotated anisotropic Gaussian, a symmetric Gaussian mixture, and a standardized bivariate $t_3$ distribution.\footnote{The four designs were: (i) $\mathcal N(0,I_2)$; (ii) $\mathcal N(0,\Sigma)$ with $\Sigma=Q\operatorname{diag}(1.8,0.2)Q^\top$ and $Q$ a rotation by $\pi/6$; (iii) the equal mixture of $\mathcal N(+1.25e_1,0.6^2I_2)$ and $\mathcal N(-1.25e_1,0.6^2I_2)$, whose unscaled covariance trace is $1.25^2+2(0.6^2)=2.2825$, followed by multiplication of the complete draw by $\sqrt{2/2.2825}=0.936072754401818$ to give trace $2$; and (iv) an isotropic bivariate $t_3$ distribution with scale matrix $I_2/3$.}  The fixed grid was $\beta\in\{1.1,1.5,1.9\}$ and $q_{\rm vote}\in\{0.10,0.20,0.50\}$, with each cloud reused across the $\beta/q_{\rm vote}$ settings.  The threshold target determines
\[
k=\max\{1,\lfloor q_{\rm vote}m+1/2\rfloor\},
\]
whose corresponding strict-conformal $\varepsilon$ interval is $[k/(m+1),(k+1)/(m+1))$.  Certificate summaries therefore operate at the 480 cloud--$\beta$ level, whereas witness and radial summaries use repeated $\beta/q_{\rm vote}$ configurations within cloud; the resulting 1,440 cloud--$\beta$--$k$ rows are not independent observations.  Values at 16,384 directions occur only for cases that triggered the prespecified fallback rule in Appendix~\ref{app:numerical-protocol}.

Across the 480 cloud--$\beta$ analyses, the median ratio $\widetilde L/L_{\rm old}$ was $0.966$, with interquartile range (IQR) approximately $[0.941,0.976]$, equivalent to a median tightening of $3.39\%$.  Descriptively, the median ratio was $0.941$ (IQR $[0.932,0.955]$) for $m=10$ and $0.975$ (IQR $[0.972,0.980]$) for $m=25$; no inferential significance is attached to this difference.  Consistent with the deterministic ordering in Proposition~\ref{prop:radial-regularity}, $\widetilde L\le L_{\rm old}$ in every one of the 480 analyses.  Figure~\ref{fig:empirical_certificate} summarizes the certificate comparison.

To fix the numerical scale, on the full validation mesh define
\[
W_{\max}(L,N)
=
\max_u\left[\overline R_{L,N}(u)-\underline R_{L,N}(u)\right].
\]
On the 8,192-direction radial mesh, with radial values $\widehat R_i$, define
\[
A_\Gamma=\frac{\pi}{8192}\sum_{i=1}^{8192}\widehat R_i^2,
\qquad
r_{\rm eq}=\sqrt{A_\Gamma/\pi}.
\]
For the operational target $W_{\max}\le0.10\,r_{\rm eq}$, neither certificate achieved the target at $N\le128$ in the expanded study.  At $N=256$, the descriptive fraction was $128/1440=8.89\%$ for $L_{\rm old}$ and $165/1440=11.46\%$ for $\widetilde L$.  These fractions aggregate repeated cloud--$\beta$--$k$ configurations and are not independent-observation estimates.  The tilde envelope width was smaller in all 7,200 declared analysis-by-$N$ comparisons, as expected from the deterministic certificate ordering.

\begin{figure}[tbp]
\centering
\includegraphics[width=\linewidth]{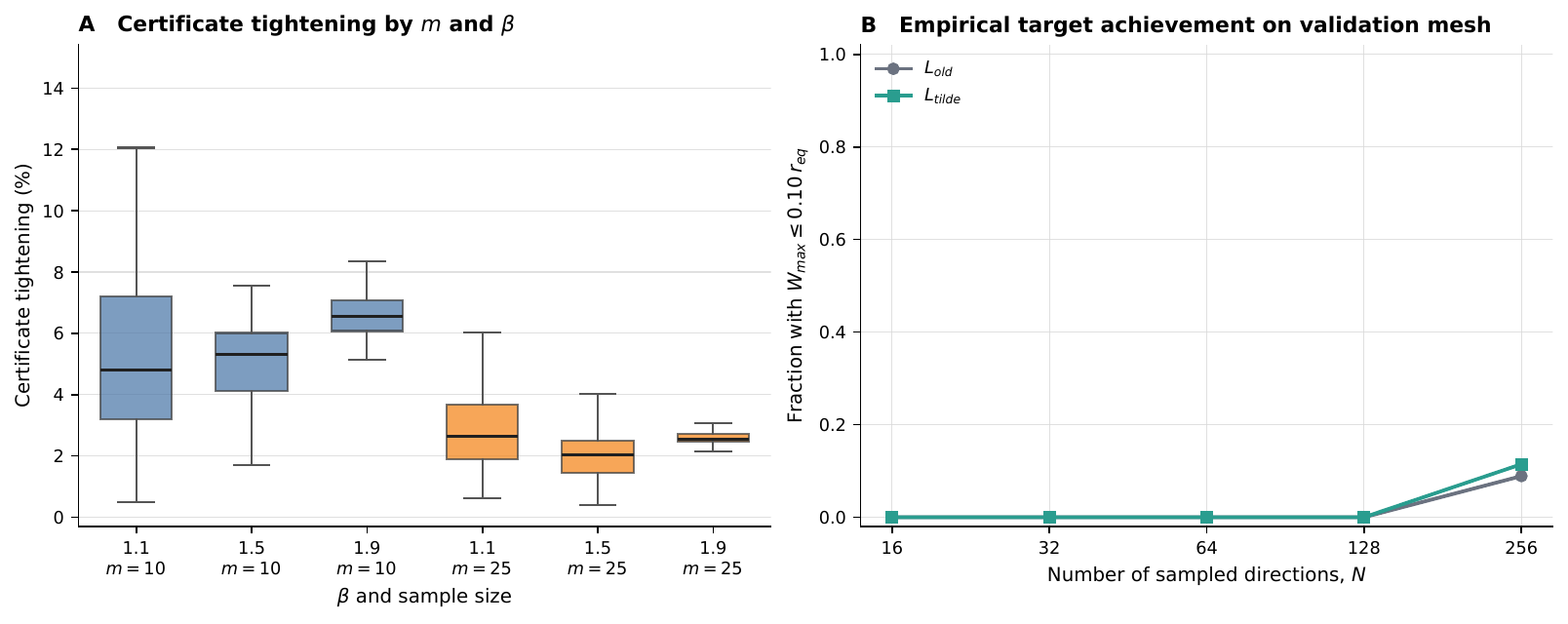}
\caption{%
Empirical behavior of the radial certificates in the staged two-dimensional study.  Panel A reports the distribution of $100(1-\widetilde L/L_{\rm old})$ at the cloud--$\beta$ level.  Panel B reports the descriptive fraction of cloud--$\beta$--$k$ configurations for which the full validation-mesh envelope width satisfies $W_{\max}\le0.10\,r_{\rm eq}$.  Repeated $\beta/q_{\rm vote}$ configurations within a cloud are not independent observations.
}
\label{fig:empirical_certificate}
\end{figure}

For the frozen 4,096-direction diagnostic mesh, the radial brackets are $\widehat R_i\pm e_i$, and we write the finite-grid lower proxy as
\[
L_{\rm sec}^{-}
=
\max_{i\ne j}
\frac{\max\{0,|\widehat R_i-\widehat R_j|-e_i-e_j\}}
{\|u_i-u_j\|}.
\]
This is a finite-grid lower proxy, not the true Lipschitz constant.\footnote{The frozen implementation label for this quantity was $L_{\mathrm{sec,all,phaseA}}^{-}$.}  The median certificate-to-finite-grid-lower-proxy ratio was $108.8$ (IQR $[63.5,201.4]$) for $L_{\rm old}$ and $105.9$ (IQR $[59.9,193.9]$) for $\widetilde L$.  Substantial residual slack remains relative to the finite-grid lower proxy.  The focuswise tightening therefore removes only part of the gap visible against this numerical lower diagnostic.

Figure~\ref{fig:empirical_nonconvexity} reports robust-witness detection rates over the prespecified grid.  Operationally, the detector forms witness hulls from the common center and genuinely accepted radial lower endpoints, inspects hull edges joining nonadjacent angular samples, and searches convex combinations after any required inward endpoint backoff; \texttt{ROBUST\_WITNESS} requires both endpoints to be robustly accepted and the convex-combination point to be robustly rejected under the stored aggregate-margin intervals, with a component/direct radial cross-check.  \texttt{SUGGESTIVE} records a prescribed candidate that remains unresolved after the frozen refinement procedure, whereas \texttt{NONE\_DETECTED} means only that this finite diagnostic found no witness; Appendix~\ref{app:numerical-protocol} gives the search details.  Each fixed $\beta/q_{\rm vote}$ cell contains exactly 160 independent clouds, and the cellwise detection rates range from $40.0\%$ to $81.2\%$.  At the independent-cloud level, $155/160$ clouds yielded at least one robust witness somewhere among the nine prespecified $\beta/q_{\rm vote}$ configurations.  These observed detection frequencies do not estimate the population prevalence of nonconvexity.

\begin{figure}[tbp]
\centering
\includegraphics[width=\linewidth]{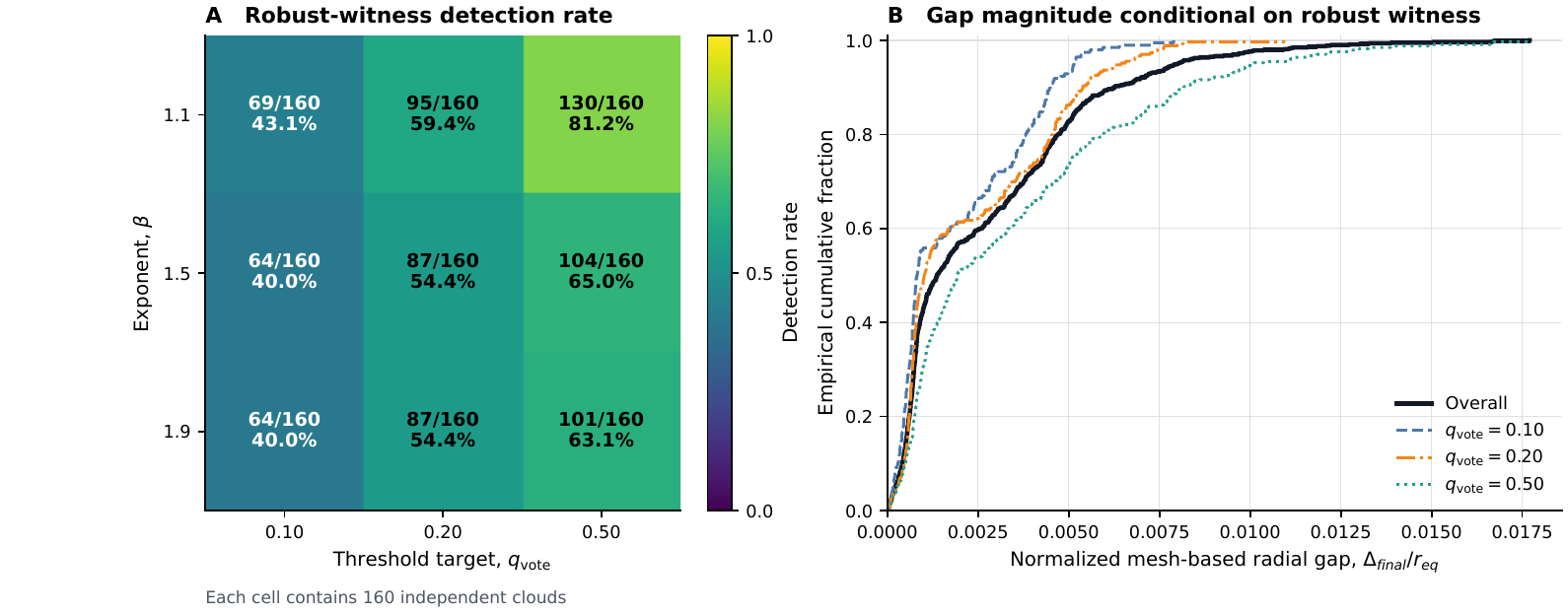}
\caption{%
Nonconvexity diagnostics in the staged two-dimensional study.  Panel A reports robust-witness detection rates, with 160 independent clouds in each fixed $\beta/q_{\rm vote}$ cell.  Panel B gives the empirical cumulative distribution function (CDF) of the normalized mesh-based radial convexification-gap approximation conditional on a robust witness.  \texttt{NONE\_DETECTED} does not certify convexity.
}
\label{fig:empirical_nonconvexity}
\end{figure}

For the reported magnitude, form the metric hull
\[
H_M=\operatorname{conv}\{c+\widehat R_i u_i\}
\]
and define the mesh-based radial convexification-gap approximation
\[
\widehat\Delta_{\rm rad}
=
\max_u\left[R_{H_M}(u)-\widehat R(u)\right]_+,
\]
reported after normalization by $r_{\rm eq}$.  Conditional on \texttt{ROBUST\_WITNESS}, this normalized approximation had median $0.00137$, 90th percentile $0.00626$, and maximum $0.01770$.  These summaries use the final gap convention.  At baseline, the fallback is triggered when $\max\{|\widehat\Delta_A-\widehat\Delta_{8192}|,|\widehat\Delta_B-\widehat\Delta_{8192}|\}/r_{\rm eq}>10^{-4}$ for the two interlaced 4,096-direction submeshes; the nested 16,384-direction refinement is then automatic, and its normalized 8,192-versus-16,384 discrepancy is checked against the same threshold.  Exactly three cases triggered this fallback and none remained unresolved; Appendix~\ref{app:numerical-protocol} gives the frozen rule.  In this staged numerical study, nonconvexity was frequently detectable, while its normalized radial magnitude was typically small.

\paragraph{Post-review Monte Carlo uncertainty summaries from the frozen outputs.}
No simulation was rerun.  A deterministic design-stratified cloud bootstrap uses only the frozen certificate and analysis CSVs: within each of the eight $(m,\mathrm{DGP})$ strata, cloud IDs are sorted and 20 whole clouds are resampled with replacement, retaining all their $\beta/q_{\rm vote}$ rows.  With \path{numpy.random.default_rng(20261001)} and 20,000 replicates, the same cloud resample computes all requested summaries in each replicate and preserves the balanced 20-per-stratum design.  The 2.5th and 97.5th percentiles for the median $\widetilde L/L_{\rm old}$ are $[0.96380,0.96909]$ overall, $[0.93904,0.94371]$ for $m=10$, and $[0.97457,0.97609]$ for $m=25$.  Table~\ref{tab:r1-mc} gives the detection-fraction limits.  These quantify finite Monte Carlo resolution of the fixed staged design; they are not population prevalence or coverage confidence intervals and imply no formal sample-size sufficiency or power claim.

\begin{table}[H]
\centering\small
\begin{tabular}{cccc}
\toprule
$\beta$ & $q_{\rm vote}=0.10$ & $q_{\rm vote}=0.20$ & $q_{\rm vote}=0.50$\\
\midrule
1.1 & $[35.6\%,50.6\%]$ & $[51.9\%,66.9\%]$ & $[75.0\%,86.9\%]$\\
1.5 & $[32.5\%,47.5\%]$ & $[46.9\%,61.9\%]$ & $[57.5\%,71.9\%]$\\
1.9 & $[32.5\%,47.5\%]$ & $[46.9\%,61.9\%]$ & $[55.6\%,70.6\%]$\\
\bottomrule
\end{tabular}
\caption{Post-review design-stratified cloud-bootstrap percentile limits for robust-witness detection fractions, computed from frozen outputs without rerunning the simulation.}
\label{tab:r1-mc}
\end{table}

\section{Computational Scope and Endpoint Limitations}
\label{sec:limitations}

Theorem~\ref{thm:radial-envelopes} and Corollary~\ref{cor:set-certification} reduce global certification to certified values on a finite directional set and a sphere-covering radius.  For chord-metric covering of $\mathbb S^{d-1}$, the fixed-dimension small-radius law is
\[
N_d(h)=\Theta_d\!\left(h^{-(d-1)}\right),
\qquad h\downarrow0,
\]
consistent with standard spherical covering and packing exponents \citep{bhattacharya_freund_jaiswal_2022_kmeans}.  If a target radial-envelope width $\tau$ satisfies $\tau>\delta$ and $\widetilde L>0$, the deterministic bound
\[
\delta+2\widetilde Lh\le\tau
\]
is guaranteed whenever
\[
h\le\frac{\tau-\delta}{2\widetilde L}.
\]
Consequently, for fixed $d$ and fine resolution,
\[
N_d
=
\Theta_d\!\left[
\left(\frac{2\widetilde L}{\tau-\delta}\right)^{d-1}
\right].
\]
If $\widetilde L=0$, the certified width is already at most $\delta$.  This fixed-$d$ notation is not a dimension-growth statement; for fixed small $h$,
\[
\log N_d(h)
=
(d-1)\log(1/h)+O(d).
\]
The directional burden is therefore a substantive limitation and confines the intended certification regime to low-dimensional multivariate outputs.

The endpoint $\beta=1$ separates exact geometry from unconditional certification.  Corollary~\ref{cor:power} and Theorem~\ref{thm:common-center} still give common-center star-shaped geometry, but the global radial function need not be continuous or Lipschitz; Appendix A gives the exact segment example.  Proposition~\ref{prop:radial-regularity} is therefore stated unconditionally only for $1<\beta<2$.  At the separate degeneracy $m=1$, $I_1=\mathbb R^d$, so finite power-distance radial reconstruction begins at $m\ge2$.

The upper restriction $\beta<2$ also keeps the certified development within the strictly proper range of the conventional energy score, but the deterministic and scoring-rule roles should not be conflated.  Convex comparison geometry holds for every $\beta\ge1$, including $\beta\ge2$; strict propriety has the separate boundaries summarized in Section~\ref{sec:radial-interface}.  Adaptive response-dependent partition issues remain outside scope, and the paper makes no new coverage theorem or runtime-superiority claim.

The numerical assessment is limited to $d=2$, four prespecified DGPs, and a finite $\beta/q_{\rm vote}$ grid.  \texttt{NONE\_DETECTED} does not certify convexity, the observed detection frequencies are not prevalence estimates, and the finite-grid secant statistic is only a lower proxy for the true Lipschitz constant.

\section{Discussion}
\label{sec:discussion}

The statistical distinction driving the paper is between validity of an exact conformal set and faithful representation of that set.  Standard conformal theory supplies the former.  For the energy-form score studied here, the paper supplies the latter structure: exact leave-one-out comparisons become pairwise-dissimilarity sublevel sets with a common center, and convexity of those comparison sets makes the nontrivial vote aggregate star-shaped without forcing it to be convex.  The univariate empirical-CRPS interval result is one specialization of this mechanism, while the analytic example shows why convexification can change the exact accepted set even when inherited coverage validity is not in question.

For $1<\beta<2$, the star-shaped geometry also supports deterministic reconstruction guarantees.  Explicit radial regularity turns finite directional brackets into global radial and Hausdorff certificates, with the formal dominance $\widetilde L\le L_{\rm old}$ preserved.  In the frozen numerical study the median certificate tightening was approximately $3.4\%$: systematic but modest, with both certificates remaining roughly two orders above the finite-grid lower proxy and hence substantial residual certificate slack.  Robust nonconvexity witnesses were often detected but typically shallow in normalized radial magnitude.  These observations are descriptive for the frozen staged two-dimensional designs.

The analytic example is a stress-test and existence illustration of materially nonfaithful convexification, not a representative-magnitude claim.  At the smaller observed numerical gaps, convexification may be negligible for coarse visualization or coarse summaries, while exact boundary-near candidate membership can still change.  The results remain score-specific, low-dimensional, and separate from adaptive-partition coverage questions or runtime claims.

These results concern response-only energy-form FullCP. In regression-based FullCP with candidate-dependent refitting, the trial response can alter the fitted predictor and hence the score comparisons, so the common-center reduction proved here does not automatically extend to that setting. Characterizing analogous geometry under fitted-model constructions is outside the present scope.

\appendix

\section{Exact \texorpdfstring{$\beta=1$}{beta=1} Segment Pathology}

Let $d\ge2$, $m=3$, and
\[
y_1=-e_1,
\qquad
y_2=0,
\qquad
y_3=e_1,
\qquad
\beta=1.
\]
Take
\[
k=3,
\]
which is equivalent under the strict convention to
\[
\frac34\le\varepsilon<1.
\]
Since all three comparison votes are required,
\[
\Gamma^\varepsilon=I_1\cap I_2\cap I_3.
\]
The middle comparison region is
\[
I_2
=
\left\{
 t:\|t+e_1\|+\|t-e_1\|\le2
\right\}.
\]
By the triangle inequality,
\[
\|t+e_1\|+\|t-e_1\|
\ge\|2e_1\|=2,
\]
and equality holds precisely on the segment joining the two foci.  Hence
\[
I_2=[-e_1,e_1].
\]
For $t=se_1$ with $-1\le s\le1$,
\[
\|t\|+\|t-e_1\|\le3,
\qquad
\|t\|+\|t+e_1\|\le3,
\]
so the whole segment also lies in $I_1$ and $I_3$.  Therefore
\[
\Gamma^\varepsilon=[-e_1,e_1].
\]

The Fr\'echet objective is minimized at $c=0$.  About this common center,
\[
R_\varepsilon(e_1)=R_\varepsilon(-e_1)=1,
\]
whereas for every direction $u$ not collinear with $e_1$ the ray $\{ru:r>0\}$ leaves the segment immediately, so
\[
R_\varepsilon(u)=0.
\]
Thus the radial function is discontinuous on $\mathbb S^{d-1}$ when $d\ge2$.  Exact common-center star-shaped geometry at $\beta=1$ therefore does not imply an unconditional global Lipschitz radial certificate.

\section{A Three-Point Nonconvex Robustness Witness}

Consider
\[
d=2,
\qquad
m=3,
\qquad
\beta=\frac32,
\]
\[
y_1=(-1,0),
\qquad
y_2=(1,0),
\qquad
y_3=(0,1.5),
\]
and
\[
\varepsilon=0.6,
\qquad
k=2.
\]
A common Fr\'echet minimizer is
\[
c=(0,0.5247333004905787534\ldots).
\]
Let
\[
A=(0.58,-0.94),
\qquad
B=(-0.58,-0.94),
\qquad
M=(0,-0.94).
\]
For $G_j(t)=\sum_{i\ne j}\|t-y_i\|^{3/2}$ and the corresponding $d_j$, the resulting comparison margins and vote counts are
\[
\begin{array}{c|rrr|c}
 & d_1-G_1 & d_2-G_2 & d_3-G_3 & N(\cdot)\\
\hline
A&
+0.232492290498&
-1.215637558202&
+1.303611687945&
2\\
B&
-1.215637558202&
+0.232492290498&
+1.303611687945&
2\\
M&
-0.170268407563&
-0.170268407563&
+1.625412619066&
1
\end{array}
\]
Because $k=2$ and equality would count as a vote, these strict numerical margins imply
\[
A,B\in\Gamma^\varepsilon,
\qquad
M\notin\Gamma^\varepsilon.
\]
The witness therefore supplies an accepted-endpoints/rejected-midpoint nonconvexity check with nonzero margins.

\section{Technical Proof Details}

\begin{table}[H]
\centering\small
\begin{tabular}{p{0.18\linewidth}p{0.74\linewidth}}
\toprule
Symbol & Definition\\
\midrule
$D_{ij}$ & $H_j+\|c-y_i\|$\\
$U_{ij}$ & $\min\{r_j+\|\mu_{-j}-y_i\|,d_j^{1/\beta}\}$\\
$M_j$ & $\beta\sum_{i\ne j}D_{ij}^{\beta-1}$\\
$\lambda_j$ & $\beta(\beta-1)\sum_{i\ne j}D_{ij}^{\beta-2}$\\
$\ell_j$ & $2M_j/\lambda_j$ for $d_j>0$; $0$ for $d_j=0$\\
$\widetilde M_j$ & $\beta\sum_{i\ne j}U_{ij}^{\beta-1}$\\
$\widetilde\lambda_j$ & $\beta(\beta-1)\sum_{i\ne j}U_{ij}^{\beta-2}$\\
$\widetilde\ell_j$ & $2\widetilde M_j/\widetilde\lambda_j$ for $d_j>0$; $0$ for $d_j=0$\\
$L_{\rm old}$ & $\max_j\ell_j$\\
$\widetilde L$ & $\max_j\widetilde\ell_j$\\
$L_{\rm sec}^{-}$ & $\max_{i\ne j}\max\{0,|\widehat R_i-\widehat R_j|-e_i-e_j\}/\|u_i-u_j\|$\\
$r_{\rm eq}$ & $\sqrt{A_\Gamma/\pi}$; numerical $A_\Gamma=(\pi/8192)\sum_{i=1}^{8192}\widehat R_i^2$\\
$\widehat\Delta_{\rm rad}$ & $\max_u[R_{H_M}(u)-\widehat R(u)]_+$; frozen final-gap refinement convention as in Section~\ref{sec:empirical_geometry}\\
\bottomrule
\end{tabular}
\caption{Symbol guide for the certificates and numerical diagnostics.  The $D_{ij},U_{ij},M_j,\lambda_j,\widetilde M_j,\widetilde\lambda_j$ formulas apply to the $d_j>0$ branch; the singleton branch evaluates no negative power at zero.}
\label{tab:r1-symbols}
\end{table}

For a non-singleton comparison set $I_j$, let finite positive constants $V_{ij}$ satisfy
\[
\sup_{t\in I_j}\|t-y_i\|\le V_{ij},
\qquad i\ne j,
\]
and define
\[
M_j(V)=\beta\sum_{i\ne j}V_{ij}^{\beta-1},
\qquad
\lambda_j(V)=\beta(\beta-1)\sum_{i\ne j}V_{ij}^{\beta-2}.
\]
The calculations below prove the modular bound
\[
\operatorname{Lip}(R_j)\le\frac{2M_j(V)}{\lambda_j(V)}.
\]
The explicit main-text choice is $V_{ij}=U_{ij}$; singleton comparison sets are handled by the constant-zero radial argument.

This appendix is the unique detailed location for the technical calculations underlying Section~\ref{sec:certified-reconstruction}.

\subsection*{Jensen containment and the radius $H_j$}

For $\beta\ge1$, the map $z\mapsto\|t-z\|^\beta$ is convex.  Applying Jensen's inequality to the $m-1$ points $\{y_i:i\ne j\}$ gives
\[
\frac1{m-1}\sum_{i\ne j}\|t-y_i\|^\beta
\ge
\left\|t-\frac1{m-1}\sum_{i\ne j}y_i\right\|^\beta
=
\|t-\mu_{-j}\|^\beta.
\]
If $t\in I_j$, then $G_j(t)\le d_j$, so
\[
\|t-\mu_{-j}\|
\le
\left(\frac{d_j}{m-1}\right)^{1/\beta}
=r_j.
\]
Thus $I_j\subseteq B(\mu_{-j},r_j)$.  Since $c\in I_j$, every $t\in I_j$ also satisfies
\[
\|t-c\|
\le
\|t-\mu_{-j}\|+\|\mu_{-j}-c\|
\le
r_j+\|c-\mu_{-j}\|
=H_j.
\]
If $H_j=0$, then both nonnegative summands in its definition vanish, $r_j=0$, and the Jensen ball reduces to $\{c\}$, proving $I_j=\{c\}$.

If $d_j>0$, then for every $t\in I_j$ and $i\ne j$, the same containment and the triangle inequality give
\[
\|t-y_i\|
\le
r_j+\|\mu_{-j}-y_i\|.
\]
Also, the $i$th nonnegative summand in $G_j(t)\le d_j$ gives
\[
\|t-y_i\|\le d_j^{1/\beta}.
\]
Thus $\|t-y_i\|\le U_{ij}$.  Both branches defining $U_{ij}$ are finite and strictly positive when $d_j>0$.

\subsection*{Focuswise derivative bounds with restricted-domain crossings}

Let $I_j$ be non-singleton, and let finite positive $V_{ij}$ bound $\|t-y_i\|$ for $t\in I_j$.  For $t\ne y_i$,
\[
\nabla\|t-y_i\|^\beta
=
\beta\|t-y_i\|^{\beta-2}(t-y_i),
\]
so
\[
\|\nabla\|t-y_i\|^\beta\|
=
\beta\|t-y_i\|^{\beta-1}.
\]
At a focus the gradient is zero by continuity.  Therefore, for every $t\in I_j$,
\[
\|\nabla G_j(t)\|
\le
\beta\sum_{i\ne j}V_{ij}^{\beta-1}
=M_j(V).
\]

Away from a focus, the Hessian is
\[
\nabla^2\|t-y_i\|^\beta
=
\beta\|t-y_i\|^{\beta-2}I
+
\beta(\beta-2)\|t-y_i\|^{\beta-4}(t-y_i)(t-y_i)^\top.
\]
Its eigenvalue in the radial direction is
\[
\beta(\beta-1)\|t-y_i\|^{\beta-2},
\]
and every orthogonal eigenvalue is
\[
\beta\|t-y_i\|^{\beta-2}.
\]
Hence the minimum eigenvalue is the former.  Since $\beta-2<0$ and $\|t-y_i\|\le V_{ij}$ on $I_j$,
\[
\nabla^2\|t-y_i\|^\beta
\succeq
\beta(\beta-1)V_{ij}^{\beta-2}I
\]
whenever $t\in I_j$ and $t\ne y_i$.

To justify the integrated inequality when a segment crosses one or more foci, take arbitrary $x,y\in I_j$ and write $z(s)=x+s(y-x)$ for $s\in[0,1]$.  Convexity of $I_j$ is essential here: it ensures that $z(s)\in I_j$ for the entire segment.  If $x=y$, the desired inequality is immediate.  Otherwise, a nonconstant line segment can meet each fixed observation at at most one parameter value, so the set of focus-crossing parameters is finite.  Partition $[0,1]$ at these values.  On every open subinterval free of crossings, the ordinary Hessian exists and, for
\[
\phi(s)=G_j(z(s)),
\]
one has
\[
\phi''(s)
=(y-x)^\top\nabla^2G_j(z(s))(y-x)
\ge
\lambda_j(V)\|y-x\|^2.
\]
The function $\|\cdot-y_i\|^\beta$ is $C^1$ for $\beta>1$, with gradient $0$ at the focus, so $\phi'$ is continuous across every partition point.  The second-derivative singularity has order $|s-s_0|^{\beta-2}$ and is locally integrable because $\beta-2>-1$.  Integrating the lower bound on the open pieces and summing across the crossings therefore gives
\[
\phi'(s)-\phi'(0)
\ge
\lambda_j(V)s\|y-x\|^2
\qquad(0\le s\le1),
\]
and a second integration yields
\[
\phi(1)
\ge
\phi(0)+\phi'(0)+\frac{\lambda_j(V)}{2}\|y-x\|^2.
\]
Since $\phi'(0)=\nabla G_j(x)^\top(y-x)$, this is precisely
\[
G_j(y)
\ge
G_j(x)
+
\nabla G_j(x)^\top(y-x)
+
\frac{\lambda_j(V)}{2}\|y-x\|^2.
\]
This argument avoids treating the classical Hessian as a finite matrix at $t=y_i$ and proves the inequality only for $x,y\in I_j$.  Taking $V_{ij}=U_{ij}$ gives $M_j(V)=\widetilde M_j$ and $\lambda_j(V)=\widetilde\lambda_j$.

\subsection*{Tangent-ball algebra and uniformization}

Let $x\in\partial I_j$ for a non-singleton $I_j$.  Strict convexity of $G_j$ implies that a non-singleton sublevel has nonempty interior.  The restricted strong-convexity inequality also shows that $g_x:=\|\nabla G_j(x)\|>0$: if $\nabla G_j(x)=0$, applying it to any distinct $y\in I_j$ would give $G_j(y)>G_j(x)=d_j$, a contradiction.  Set
\[
p_x=\nabla G_j(x)/g_x.
\]
For $y\in I_j$, the strong-convexity inequality and $G_j(y)\le G_j(x)$ imply
\[
g_x p_x^\top(y-x)+\frac{\lambda_j(V)}{2}\|y-x\|^2\le0.
\]
Multiplying by $2/\lambda_j(V)$ and completing the square gives
\begin{align*}
\|y-x\|^2
+2\frac{g_x}{\lambda_j(V)}p_x^\top(y-x)
&\le0,\\
\left\|y-x+\frac{g_x}{\lambda_j(V)}p_x\right\|^2
&\le
\left(\frac{g_x}{\lambda_j(V)}\right)^2.
\end{align*}
Thus
\[
I_j
\subset
B\left(x-\frac{g_x}{\lambda_j(V)}p_x,\frac{g_x}{\lambda_j(V)}\right).
\]
Writing $r=g_x/\lambda_j(V)$ and $R=M_j(V)/\lambda_j(V)$, one has $r\le R$.  The centers of
\[
B(x-rp_x,r)
\quad\text{and}\quad
B(x-Rp_x,R)
\]
are separated by $R-r$, so every point of the first ball lies within $(R-r)+r=R$ of the second center.  Hence
\[
B(x-rp_x,r)\subseteq B(x-Rp_x,R),
\]
which proves the uniformized radius.

For completeness, the normal-cone identity used in the main text follows from the standard convex differentiable sublevel rule.  The nonempty interior supplies strict feasibility for the sublevel, and the boundary gradient is nonzero, so
\[
N_{I_j}(x)=\{\tau\nabla G_j(x):\tau\ge0\}.
\]

The exact uniform-radius intersection statement used in the main proof follows from metric projection.  Let
\[
R=\frac{M_j(V)}{\lambda_j(V)}
\]
and consider the family
\[
\mathcal B_j
=
\{B(x-Rp_x,R):x\in\partial I_j\}.
\]
The tangent-ball calculation above gives
\[
I_j\subseteq\bigcap_{B\in\mathcal B_j}B.
\]
Conversely, for $z\notin I_j$, let $x=\Pi_{I_j}(z)$.  Then
\[
\frac{z-x}{\|z-x\|}\in N_{I_j}(x),
\]
and hence equals $p_x$ by the normal-cone identity after normalization.  The corresponding ball belongs to $\mathcal B_j$ but excludes $z$, because
\[
\left\|z-(x-Rp_x)\right\|
=
R+\|z-x\|
>
R.
\]
Thus
\[
I_j=\bigcap_{B\in\mathcal B_j}B.
\]

\subsection*{The factor-$2$ radial Lipschitz estimate}

Translate the common radial center to the origin.  Consider a ball $B(a,R)$ with $\|a\|\le R$, so the origin lies in the ball.  A point $ru$ on the positive ray is on the boundary when
\[
\|ru-a\|^2=R^2,
\]
or
\[
r^2-2r(a^\top u)+\|a\|^2-R^2=0.
\]
The nonnegative exit root is
\[
r_a(u)
=
a^\top u+\sqrt{(a^\top u)^2+R^2-\|a\|^2}.
\]
Let $C=R^2-\|a\|^2$.  For $C>0$, the scalar function
\[
\psi(s)=s+\sqrt{s^2+C}
\]
has derivative
\[
\psi'(s)=1+\frac{s}{\sqrt{s^2+C}}\in(0,2).
\]
For $C=0$, $\psi(s)=s+|s|$ is globally $2$-Lipschitz.  Therefore
\[
|r_a(u)-r_a(v)|
\le2|a^\top(u-v)|
\le2\|a\|\,\|u-v\|
\le2R\|u-v\|.
\]
If $\{f_\xi\}$ are all $K$-Lipschitz, let
\[
g(u)=\inf_\xi f_\xi(u).
\]
For any $\eta>0$, choose $\xi_\eta$ such that
\[
f_{\xi_\eta}(v)\le g(v)+\eta.
\]
Then
\[
g(u)-g(v)
\le
f_{\xi_\eta}(u)-f_{\xi_\eta}(v)+\eta
\le
K\|u-v\|+\eta.
\]
Letting $\eta\downarrow0$ and then exchanging $u$ and $v$ gives
\[
|g(u)-g(v)|\le K\|u-v\|.
\]
Hence an arbitrary infimum of functions sharing the same Lipschitz constant retains that constant.  By the exact supporting-ball representation established above and in the main proof, $R_j$ is the infimum of the corresponding radius-$R$ ball radial exits.  Applying the preceding Lipschitz bound therefore gives
\[
\operatorname{Lip}(R_j)
\le2R
=\frac{2M_j(V)}{\lambda_j(V)}.
\]
For $d_j>0$, taking $V_{ij}=U_{ij}$ gives the tightened bound $\operatorname{Lip}(R_j)\le\widetilde\ell_j$.  Singleton comparison sets use the separate constant-zero radial argument instead.

\subsection*{Order-statistic, interval-envelope, and Hausdorff calculations}

Let $a,b\in\mathbb R^m$ satisfy $\|a-b\|_\infty\le\eta$.  Then $a_j\le b_j+\eta$ for all $j$.  If $a_{(r)}>b_{(r)}+\eta$, at least $m-r+1$ coordinates of $a$ would exceed $b_{(r)}+\eta$, forcing the corresponding coordinates of $b$ to exceed $b_{(r)}$, a contradiction.  Thus $a_{(r)}\le b_{(r)}+\eta$; exchanging $a,b$ yields
\[
|a_{(r)}-b_{(r)}|\le\eta.
\]
This proves the $\ell_\infty$ stability used for $R_\varepsilon$.

For the interval envelopes, Lipschitzness and the sampled bounds give, for every $\ell$,
\[
a_\ell-\widetilde L\|u-u_\ell\|
\le
R_\varepsilon(u)
\le
b_\ell+\widetilde L\|u-u_\ell\|.
\]
Taking the maximum of the lower bounds together with the intrinsic constraint $R_\varepsilon\ge0$, and the minimum of the upper bounds, gives $\underline R\le R_\varepsilon\le\overline R$.  Choosing a nearest sampled direction and splitting according to whether $a_\ell-\widetilde L\|u-u_\ell\|$ is nonnegative gives
\[
\overline R(u)-\underline R(u)
\le\delta+2\widetilde Lh_{\mathcal U},
\]
as in the main proof.

Finally, let $A$ and $B$ be same-center radial sets with radial functions $R_A\le R_B$.  Every $x=c+ru\in B$ is within
\[
\max\{0,r-R_A(u)\}
\le
R_B(u)-R_A(u)
\]
of the point $c+\min\{r,R_A(u)\}u\in A$.  Since $A\subseteq B$, the reverse directed distance is zero, and therefore
\[
d_H(A,B)
\le
\|R_B-R_A\|_\infty.
\]
Applying this elementary same-ray bound to $(\underline\Gamma,\Gamma^\varepsilon)$, $(\Gamma^\varepsilon,\overline\Gamma)$, and $(\underline\Gamma,\overline\Gamma)$ gives all three Hausdorff inequalities in Corollary~\ref{cor:set-certification}.

\section{Numerical Protocol and Reproducibility Details}
\label{app:numerical-protocol}

This appendix records only the frozen numerical procedures needed to make Section~\ref{sec:empirical_geometry} operationally inspectable; no numerical study was rerun or redesigned for this description.

\subsection*{Robust nonconvexity witness diagnostic}
For each declared reference resolution (phase A, phase B, and the full 8,192-direction mesh), let $a_i$ be the genuinely accepted lower radial endpoint on direction $u_i$.  The witness hull is
\[
H_M^{\rm in}=\operatorname{conv}\bigl(\{c\}\cup\{c+a_i u_i\}\bigr).
\]
Only hull edges joining nonadjacent circular angular samples are inspected.  For such an edge with endpoints $x,y$, the frozen search first evaluates $\lambda\in\{1/4,1/2,3/4\}$, then the systematic grid $\lambda=j/128$, $j=1,\ldots,127$, and then refines the cell around the smallest or unresolved aggregate margin until its $\lambda$ width is at most $2^{-20}$ or 20 refinement levels have been reached.  If a boundary-near endpoint is not robustly accepted, it is moved only inward on its original ray as $x(\rho)=c+(1-\rho)(x-c)$, with ascending $\rho\in\{10^{-12},10^{-11},\ldots,10^{-2},0.02,0.05,0.10\}$; the point $z=(1-\lambda)x(\rho_x)+\lambda y(\rho_y)$ is always recomputed from the backed-off endpoints and is never moved independently.

The aggregate witness margin is $M_k=q_{[k]}$, where $q_j(x)=d_j-G_j(x)$ and $q_{[k]}$ is the $k$th largest component margin.  In binary64 arithmetic, each component quantity $q_j$ is accompanied by the implementation's prescribed numerical-error bound, and these bounds are propagated through the order statistic to lower and upper screening bounds for $M_k$.  At baseline \texttt{margin\_factor}$=1$, the exact margin screen is
\[
\tau_q(x)=10^{-10}\max\{1,\max_j[d_j+G_j(x)]\}.
\]
The factor-of-ten replay uses \texttt{margin\_factor}$=0.1$.  Binary64 robust acceptance requires the lower $M_k$ screening bound to exceed $\tau_q(x)$; robust rejection requires the upper bound to be below $-\tau_q(x)$.  Otherwise the quantities are recomputed with Python standard-library \texttt{decimal} at context precision 90 decimal digits (Decimal90).  This path reconstructs coordinates with \texttt{Decimal.from\_float}, recomputes both $d_j$ and $G_j$, and does not reuse binary64 $d_j$; cached $d_j$ values are themselves Decimal90 recomputations.  Noninteger powered distances are evaluated using Decimal \texttt{ln}/\texttt{exp}.  The Decimal90 aggregate margin is compared against the same stored numeric $\tau_q$ value converted using \texttt{Decimal.from\_float}: above $\tau_q$ is robust acceptance, below $-\tau_q$ is robust rejection, and otherwise the result remains unresolved.  At a given resolution, \texttt{ROBUST\_WITNESS} requires two robustly accepted endpoints and a robustly rejected convex-combination point $z$; \texttt{SUGGESTIVE} means that no robust witness was found but at least one prescribed witness candidate remains unresolved after the elevated-precision and edge-refinement procedure; and \texttt{NONE\_DETECTED} means neither condition occurred.  The three resolution categories are aggregated in a fixed order: \texttt{ROBUST\_WITNESS} ranks above \texttt{SUGGESTIVE}, which ranks above \texttt{NONE\_DETECTED}.  Every proposed robust witness is also subjected to the component/direct radial cross-check; after the resolution-wise categories are aggregated, a would-be \texttt{ROBUST\_WITNESS} is downgraded to \texttt{SUGGESTIVE} if the component/direct radial cross-check raises \texttt{RADIAL\_CROSSCHECK\_FAILURE}.  Thus \texttt{NONE\_DETECTED} means only that no witness was found by this finite diagnostic, not that the region is convex.

\subsection*{Arithmetic scope and radial termination}
Common-center inequalities use their own binary64 error screen and a Decimal90 retry when unresolved; they do not use the witness $\tau_q$ threshold.  The mathematical reconstruction theorem is deterministic conditional on valid directional brackets and verified common-center inequalities.  The implementation uses explicit numerical screening and precision escalation for these conditions; it does not claim a general machine-level proof of every IEEE rounding operation or use a general outward-rounded interval-arithmetic library.

Figure~\ref{fig:analytic-certificate} uses exact sampled analytic radii, $\delta=0$.  In the expanded study the radii are bracketed.  At baseline \texttt{root\_factor}$=1$, direct radial termination uses
\[
\texttt{root\_tol}=\max\{10^{-12},10^{-10}\max(1,H_{\rm out})\},
\qquad
\delta_N=\max_\ell(b_\ell-a_\ell),
\]
where $H_{\rm out}$ is the common rejected direct outer-bracket radius actually used for the solve (after any prescribed bracket expansion), and $\delta_N$ is the directional uncertainty used for the envelope.  This distinguishes numerical bracket widths from exact analytic sampled radii; the root finder is unchanged.

\subsection*{Smoke-to-production expansion gates}
The 5-cloud-per-cell smoke stage could expand to 20 clouds per cell only if all seven frozen gates passed:
\begin{enumerate}
\item no unexplained hard failure remained;
\item both-start/tight-center and $L_{\rm old}$ stability diagnostics passed;
\item component/direct radial cross-checks passed;
\item every envelope satisfied its deterministic bound within numerical error;
\item phase/resolution instability was absent or explicitly resolved by a predeclared tighter setting;
\item the factor-of-ten numerical sensitivity subset was stable; and
\item runtime and storage remained practical.
\end{enumerate}
Soft flags were reported and did not automatically block expansion.

\subsection*{Production seeds and no-post-hoc rule}
The production master seed was 20260828.  Cell indices $c=0,\ldots,7$ follow the frozen order: for $m=10$ and then $m=25$, the DGP order is isotropic Gaussian, anisotropic Gaussian, symmetric mixture, and $t_3$.  From \texttt{numpy.random.SeedSequence(20260828)}, replicate $r=1,\ldots,20$ in cell $c$ uses child index
\[
5c+(r-1)\quad(r\le5),
\qquad
40+15c+(r-6)\quad(r\ge6).
\]
This preserves the original five smoke clouds per cell and adds 15 new clouds per cell.  Once production began, tolerances, seed mapping, clouds, DGPs, the $\beta/q_{\rm vote}$ grid, witness search, and grid resolution were not changed, except for the pre-authorized automatic 16,384-direction fallback described next; difficult clouds were retained rather than rerun under altered settings.

\subsection*{Radial-gap fallback and resolution check}
Let $\widehat\Delta_A$, $\widehat\Delta_B$, and $\widehat\Delta_{8192}$ denote the mesh-based radial convexification-gap approximations from the two interlaced 4,096-direction submeshes and the full 8,192-direction mesh.  At the baseline setting, \texttt{mesh\_tol}$=10^{-4}$ and the normalized stability discrepancy is exactly
\[
\frac{\max\bigl\{|\widehat\Delta_A-\widehat\Delta_{8192}|,
|\widehat\Delta_B-\widehat\Delta_{8192}|\bigr\}}{r_{\rm eq}}.
\]
If this quantity exceeds $10^{-4}$, the analysis receives \path{RADIAL_GAP_MESH_UNSTABLE} and automatically triggers the nested 16,384-direction refinement.  The subsequent normalized discrepancy $|\widehat\Delta_{16384}-\widehat\Delta_{8192}|/r_{\rm eq}$ is again compared with $10^{-4}$; if it still exceeds that threshold, \path{RADIAL_GAP_REFINEMENT_UNRESOLVED} is recorded.  In the expanded study exactly three analyses triggered the fallback and zero refinements remained unresolved.

\subsection*{Lean claim map and scope}
Table~\ref{tab:r1-lean} reproduces the existing \texttt{lean/CLAIM\_MAP.md} entries.  The frozen toolchain is Lean 4.19.0 and mathlib v4.19.0, resolved commit \path{c44e0c8ee63ca166450922a373c7409c5d26b00b}.  This is not whole-paper formal verification: Section 5 multivariate regularity/certification, adaptive partitions, and exchangeability/coverage are not verified by this development.

\begin{table}[htbp]
\centering\scriptsize
\begin{tabular}{>{\raggedright\arraybackslash}p{0.27\linewidth}>{\raggedright\arraybackslash}p{0.65\linewidth}}
\toprule
Manuscript result & Lean declaration(s)\\
\midrule
Convex candidate score does not force connected accepted set & \path{counterexample_candidateScore_convex}, \path{counterexample_acceptedSet_eq_two_rays}, \path{counterexample_acceptedSet_not_ordConnected} in \path{FullCPGeometry/Counterexample.lean}\\
Leave-one-out energy score difference and sublevel comparison & \path{energyScore_training_sub_candidate}, \path{energyScore_training_sub_candidate_real}, \path{comparisonSet_eq_energySublevel} in \path{FullCPGeometry/EnergyIdentity.lean}\\
Common comparison center & \path{energyComparisonSet_commonCenter}, \path{energyFullCP_commonCenter} in \path{FullCPGeometry/CommonCenter.lean}\\
Strict full-conformal vote threshold and star-shaped aggregate & \path{mem_acceptedSet_iff_voteThreshold}, \path{acceptedSet_eq_voteSuperlevelSet}, \path{energyFullCP_commonCenter_starConvex} in \path{FullCPGeometry/CommonCenter.lean}\\
Fixed-bin empirical-CRPS exact interval and endpoint cases & \path{completeMedianSet_eq_Icc}, \path{empiricalCRPS_fullCP_exactInterval}, \path{empiricalCRPS_fullCP_isCompact}, \path{empiricalCRPS_fullCP_eq_Icc}, \path{empiricalCRPS_fullCP_one_eq_univ} in \path{FullCPGeometry/UnivariateCRPS.lean}\\
\bottomrule
\end{tabular}
\caption{Verbatim selected-claim map, with source files in the frozen Lean project.}
\label{tab:r1-lean}
\end{table}

\section*{Data and code availability}
The journal-archived reproducibility materials are available at \url{https://doi.org/10.5281/zenodo.23530530}.  The reviewed bundle is \path{JRS_REPRO_PACKAGE_R1a.zip}; in the journal archive, its directory is \path{code/JRS_REPRO_PACKAGE_R1a}.  From that directory, the full command is \texttt{python reproduce.py}; quick validation is \texttt{python reproduce.py --quick}.  The \textbf{SUCCESSFULLY TESTED REPRODUCTION ENVIRONMENT} is Python 3.9.6, NumPy 2.0.2, pandas 2.3.3, Matplotlib 3.9.4, on macOS 26.4 arm64.  This describes the successfully tested reproduction environment, not the historical production environment.

Determinism is fixed by master seed 20260828, the fixed \texttt{numpy.random.SeedSequence} child mapping in this appendix, all 160 exact analyzed synthetic clouds included in the artifact (retaining the original 40 smoke clouds), and frozen production thresholds/configuration.  A clean full reproduction from a fresh extraction completed successfully on 30 September 2026 in 4,053.986 seconds.  The revised artifact preserves the scientific numerical pipeline; its presentation and post-review derived-analysis additions were validated separately, and quick validation passed. No new full 160-cloud production rerun was performed.

The full entry point reconstructs an equivalent analytic rendering of Figure~\ref{fig:analytic-certificate} from the frozen formulas/specification and regenerates Figures~\ref{fig:empirical_certificate}--\ref{fig:empirical_nonconvexity} from the frozen empirical outputs, with the accessibility presentation layer applied.  The original historical analytic drawing script was not recovered; the supplied Figure-2 plotting script is explicitly reproduction-only tooling.  Canonical Stage B2/B3 scientific sources, raw clouds, original reference CSVs and manifests, and Lean source/toolchain are preserved; historical Figure 4 is retained in provenance.  Lean remains a separate route: \texttt{cd lean} followed by \texttt{lake build}.

\subsection*{AI-assisted research and computational workflow}

OpenAI ChatGPT using GPT-5.6 Sol assisted with mathematical derivations and mathematical checking. OpenAI Codex using GPT-5.6 Sol with the High reasoning setting assisted with code development, implementation of the numerical experiments, and Lean formalization. The author reviewed the mathematical claims, numerical results, computational outputs, and formalization artifacts used in the article and takes full responsibility for the work. No AI system was treated as an authority for mathematical correctness. No generative-AI image tool was used to create or alter the final figures. No AI system generated or recommended the references used in the article.

\section*{Declaration of generative AI and AI-assisted technologies in the manuscript preparation process}

During the preparation of this work, the author used OpenAI ChatGPT (GPT-5.6 Sol) for language polishing and OpenAI Codex (GPT-5.6 Sol, High reasoning setting) for LaTeX preparation. After using these tools, the author reviewed and edited the content as needed and takes full responsibility for the content of the publication.

\section*{Competing interests and funding}
The author declares no competing interests. This research received no funding.

\bibliographystyle{plainnat}
\bibliography{references}

\end{document}